\documentclass{llncs}
\usepackage{pgf}
\usepackage{tikz}
\usepackage[vlined,ruled]{algorithm2e}
\usepackage{amssymb}
\usetikzlibrary{arrows,automata}

\title{A Constraint Programming Approach for Mining Sequential
  Patterns in a Sequence Database}

\author{
Jean-Philippe M\'etivier$^1$ \and Samir Loudni$^1$ \and Thierry Charnois$^2$
}
\institute{
$^1$ GREYC (CNRS UMR 6072) -- University of Caen\\ 
Campus II C\^ote de Nacre, 14000 Caen - France\\
$^2$ LIPN  CNRS (UMR 7030) -- University PARIS 13 \\  
99, avenue Jean-Baptiste Clément 93430 Villetaneuse - France
}

\begin{document}
\maketitle

\begin{abstract}
Constraint-based pattern discovery is at the core of numerous data
mining tasks. Patterns are extracted with respect to a given set of
constraints (frequency, closedness, size, etc). In the context of
sequential pattern mining, a large number of devoted techniques have
been developed for solving particular classes of constraints. The aim of
this paper is to investigate the use of Constraint Programming (CP) to model and mine 
sequential patterns in a sequence database. Our CP approach offers a
natural way to simultaneously combine in a same framework a large set
of constraints coming from various origins. 
Experiments show the feasibility and the interest of our approach.
\end{abstract}

\section{Introduction}
\label{section:introduction}
Sequential pattern mining is a well-known data mining technique
introduced in~\cite{DBLP:conf/icde/AgrawalS95} to find regularities in
a database of sequences. This problem is central in many application
domains, such as web usage mining~\cite{DBLP:conf/kdd/CadezHMSW00}, 
bioinformatics and text mining~\cite{DBLP:conf/cicling/BechetCCC12}.  
For effectiveness and efficiency considerations, many authors
\cite{DBLP:series/ads/DongP07,DBLP:conf/cikm/Zaki00} have 
promoted the use of constraint to focus on the most promising
knowledge by reducing the number of extracted patterns to those of a
potential interest given by the final user. The most popular example
is the minimal frequency constraint: it addresses all sequences having
a number of occurrences in the database exceeding a given minimal threshold. 

There are already in the literature many algorithms to extract 
{\it sequential patterns} (e.g. \texttt{GSP}
\cite{DBLP:conf/edbt/SrikantA96}, \texttt{SPADE}
\cite{DBLP:journals/ml/Zaki01}, \texttt{PrefixSpan}
\cite{DBLP:conf/icde/PeiHPCDH01}), 
{\it closed} sequential patterns (e.g. \texttt{CloSpan}
\cite{DBLP:conf/sdm/YanHA03}, \texttt{BIDE}
\cite{DBLP:conf/icde/WangH04}) or sequential patterns
satisfying {\it regular expression} (e.g. \texttt{SPIRIT} \cite{DBLP:journals/tkde/GarofalakisRS02}). 
All the above methods, though efficient, suffer from two major
problems. First, they tackle particular classes of constraints
(i.e. {\it monotonic} and {\it anti-monotonic} ones) by using
devoted techniques. However, several practical constraints required in
data mining tasks, such as regular expression, aggregates, do
not fit into these classes. Second, they lack of
generic methods to push various constraints into the sequential
pattern mining process. Indeed, adding and handling simultaneously
several types of constraints in a nice and elegant way beyond the
few classes of constraints studied is not still trivial. 
The lack of generic approaches restrains the discovery of useful
patterns because the user has to develop a new method each time he
wants to extract patterns satisfying a new type of constraints. In
this paper, we address this open issue by proposing a generic approach
for modelling and mining sequential patterns under various constraints
using Constraint Programming (CP). 

Our proposition benefits from the recent progress on
cross-fertilization between data mining and CP for itemset mining
\cite{DBLP:journals/ai/GunsNR11,CI-KHIARI-10b,DBLP:conf/kdd/RaedtGN08}.  
The common point of all these methods is to model in a
declarative way pattern mining as Constraint Satisfaction Problems
(CSP), whose resolution provides the complete set of solutions
satisfying all the constraints. The great advantage of this modelling
is its flexibility, it enables to define and to push new constraints 
without having to develop new algorithms from scratch. 

The key contribution of this paper is to propose a CP modelling for
the problem of mining sequential patterns in a sequence 
database~\cite{DBLP:conf/icde/AgrawalS95}. 
Our approach addresses in a unified framework a large set
of constraints. This includes constraints such as frequency,
closedness and size, and other constraints such as regular expressions,
gap and constraints on items. Moreover, our approach enables to
combine simultaneously different types of constraints. This leads to
the first CP-based model for discovering sequential 
patterns in a sequence database under various
constraints. Experiments on a case study on biomedical literature for
discovering gene-RD relations from PubMed articles show the feasibility and the
interest of our approach. 

This paper is organized as follows. Section \ref{section:Sequences} gives the necessary
definitions and presents the problem formulation. Section
\ref{section:CP} introduces the main principles of constraint
programming. Section \ref{section:Modeling} describes our
CP model for mining sequential patterns in a sequence database. We
review some related work in Section \ref{section:relatedwork} and
Section \ref{section:experimentations} reports in depth a
case study from the biomedical literature domain for
discovering gene-RD relations from PubMed articles. Finally, we
conclude and draw some perspectives.  

\section{Sequential Pattern Mining}
\label{section:Sequences}
In this section, we introduce sequential patterns defined by Agrawal
et al. \cite{DBLP:conf/icde/AgrawalS95}.

\subsection{Sequential Patterns}

Sequential pattern mining~\cite{DBLP:conf/icde/AgrawalS95} is a data
mining technique that aims at discovering correlations between events
through their order of appearance. Sequential pattern mining is an
important field of data mining with broad applications (e.g., biology,
marketing, security) and there are many algorithms to extract frequent
sequences~\cite{DBLP:conf/edbt/SrikantA96,DBLP:journals/ml/Zaki01,DBLP:conf/sdm/YanHA03}. 

In the context of sequential patterns extraction, a
\emph{sequence} is an ordered list of  literals called {\em
  items}.  A sequence $s$ is denoted by $\langle i_1, i_2 \ldots i_n \rangle$
where $i_k$, $1 \leq k \leq n$, is an item.
Let $s_1 = \langle i_1, i_2 \ldots i_n \rangle$ and $s_2 = \langle i'_1, 
i'_2 \ldots i'_m \rangle$ be two sequences.  $s_1$ is \emph{included}
in $s_2$ if there exist integers $1 \leq j_1 < j_2 < \ldots < j_n \leq m$
such that $i_1 = i'_{j_1}$, $i_2 = i'_{j_2}$, $\ldots$, $i_n = i'_{j_n}$.
$s_1$~is called a \emph{subsequence} of $s_2$.  $s_2$ is called a
\emph{super-sequence} of $s_1$, denoted by $s_1 \preceq s_2$. 
For example the sequence $\langle a~ b~ d~ c
\rangle$ is a super-sequence of $\langle b~ c \rangle$: $\langle b~ c
\rangle \preceq \langle a~ b~ d~ c \rangle$.
A \emph{sequence database} $SDB$ is a set of tuples $(sid, s)$, where $sid$
is a sequence identifier and $s$ a sequence. For instance, Table~\ref{tab:BDD}
represents a sequence database of four sequences.
\begin{table}[t]
	\begin{center}
		\caption{$SDB_1$: a sequence database}
		\label{tab:BDD}
		\small{
			\begin{tabular}{cc}
				\hline
				~Sequence identifier~ & ~Sequence~ \\
				\hline
				$1$ & $\langle  a\; b\; c\; d\; a \rangle$  \\
				$2$ & $\langle  d\; a\; e \rangle$ \\
				$3$ & $\langle  a\; b\; d\; c\; \rangle$ \\
				$4$ & $\langle c\; a \rangle$ \\
				\hline
			\end{tabular}
		}%small
	\end{center}
%\vspace{-0.8cm}
\end{table}
A tuple $(sid, s)$ \emph{contains} a sequence $s_1$, if $s_1 \preceq
s$.  The \emph{support} of a sequence $s_1$ in a sequence database
$SDB$, denoted $sup(s_1)$, is the number of tuples containing $s_1$ in the
database\footnote{The
  relative support is also used: \\ $sup_{SDB}(T) =
  \frac{\displaystyle |\{ (sid, s) ~ s.t. ~ (sid, s) \in SDB \wedge
    (T \preceq s) \}|}{\displaystyle |SDB|}$}.  
For example, in Table~\ref{tab:BDD}, $sup(\langle c\;  a  \rangle)~=~2$.  

A \emph{frequent sequential pattern} is a sequence such that its
support is greater or equal to a given threshold: $minsup $. The
\emph{sequential pattern mining} problem is to find the
\emph{complete} set of frequent sequential patterns with respect to a given
sequence database SDB and a support threshold $minsup$. 
%Sequential pattern mining algorithms thus extract all the frequent
%sequential patterns that appear in a sequence database.

\subsection{Sequential Pattern Mining under Constraints}
\label{sec-constraints}

%Because the set of frequent sequential patterns can be very large, there exists
%a condensed representation which eliminates redundancies without loss of
%information: \emph{closed sequential patterns}~\cite{DBLP:conf/sdm/YanHA03}.  A frequent
%sequential pattern $s$ is closed if there exists no other frequent sequential
%pattern $s'$ such that $s \preceq s'$ and $sup(s)=sup(s')$.  For instance, with
%$minsup ~=~2$, the sequential pattern $\langle b~ c\rangle$ from
%Table~\ref{tab:BDD} is not closed whereas the pattern $\langle a\; b\;
%c\rangle$ is closed.
%

%Moreover, 
In order to drive the mining process towards the user objectives and
to eliminate irrelevant patterns, one can define constraints~\cite{DBLP:series/ads/DongP07}. 
The most commonly used constraint is the frequency constraint (that
assigns a value to $minsup $). We review some of the most important
constraints for the sequential mining problem~\cite{DBLP:series/ads/DongP07}. 

\medskip
\noindent
\textbf{Closedness Constraint} The \emph{closed sequential patterns}~\cite{DBLP:conf/sdm/YanHA03} are 
a condensed representation of the whole set of sequential patterns. This condensed representation eliminates redundancies according to the frequency constraint.  A frequent
sequential pattern $s$ is closed if there exists no other frequent sequential
pattern $s'$ such that $s \preceq s'$ and $sup(s)=sup(s')$.  For instance, with
$minsup ~=~2$, the sequential pattern $\langle b~ c\rangle$ from
Table~\ref{tab:BDD} is not closed whereas the pattern $\langle a\; b\;
c\rangle$ is closed.

\medskip
\noindent
\textbf{Item constraint.} An item constraint specifies subset of items
that should or should not be present in the sequential 
patterns. 
For instance, if we impose the constraint $ C_{item} \equiv sup(p) \geq 2 \wedge (a \in p) \wedge (b \in p)$, 
three sequential patterns are mined from Table~\ref{tab:BDD}: $p_1
= \langle a\; b\rangle$, $p_2 = \langle 
a\; b\; c\rangle$ and $p_3 = \langle a\; b\; d\rangle$. 

\medskip
\noindent
\textbf{Size constraint.} The aim of this
constraint is to limit the length of the patterns, the length being 
the number of occurrences of items. The length of a pattern $p$ will
be denoted by $len(p)$. For example, if $len(p) \ge 3 \wedge sup(p)
\geq 2$, only two sequential patterns are 
extracted ($p_2$ and $p_3$).

\medskip
\noindent
\textbf{Gap constraint.} Another widespread constraint is the gap
constraint. A sequential pattern with a gap constraint $C_{gap} \equiv[M,N]$,
denoted by $p_{[M,N]}$, is a pattern such as at least $M$ items
and at most $N$ items are allowed between every two 
neighbor items, in the original sequences.
For instance, let $p_{[0,2]}=\langle c\; a \rangle$ and
$p_{[1,2]}=\langle c\; a \rangle$ be two patterns with two different gap
constraints and let us consider the sequences of Table~\ref{tab:BDD}. 
Sequences $1$ and $4$ support pattern $p_{[0,2]}$ (sequence $1$
contains one item between $(c)$ and $(a)$ whereas sequence $4$ contains no item
between $(c)$ and $(a)$).  But only Sequence $1$  supports  $p_{[1,2]}$
(only sequences with one or two items between $(c)$ and $(a)$ support  this pattern).

\medskip
\noindent
\textbf{Regular expression constraint.} A regular expression
constraint $C_{RE}$ is a constraint specified as a regular expression
over the set of items. A sequential pattern satisfies $C_{RE}$ if and 
only if the pattern is accepted by its equivalent deterministic finite
automata \cite{DBLP:journals/tkde/GarofalakisRS02}. For instance, the
two sequential patterns $\langle a\; b\; c\rangle$ and $\langle
a\; d\; c\rangle$ from Table~\ref{tab:BDD}  satisfy the regular
expression constraint $C_{RE} =a * \{bb|bc|dc\}$.

\section{Constraint Programming}
\label{section:CP}
In this section, we first introduce basic constraint programming concepts and then
present two constraints of interest: {\tt Among} and {\tt Regular}. 

\subsection{Preliminaries}

Constraint programming (CP) is a generic framework for solving
combinatorial problems modelled as Constraint Satisfaction Problems
(CSP). The key power of CP lies in its declarative
approach towards problem solving: in CP, the user specifies the set of
constraints which has to be satisfied, and the CP solver generates the
correct and complete set of solutions. In this way, the specification
of the problem is separated from the search strategy. 
%This has the advantage that different problems can be specified by merely changing
%the declarative problem specification in terms of constraints. 

A {\it Constraint Satisfaction Problem} (CSP) consists of a finite
set of variables $\mathcal{X} = \{X_1, \ldots, X_n\}$ with finite
domains $\mathcal{D} = \{D_1, \ldots, D_n\}$ such that each $D_i$ is
the set of values that can be assigned to $X_i$, together with a
finite set of constraints $\mathcal{C}$, each on a subset of
$\mathcal{X}$. A constraint $C \in \mathcal{C}$ is a subset of the
cartesian product of the domains of the variables that are in $C$. 
The goal is to find an assignment ($X_i = d_i$) with $d_i \in D_i$
for $i = 1, \ldots, n$, such that all constraints are satisfied. This
assignment is called a solution to the CSP. For a given assignment
$t$, $t[X_i]$ denotes the value assigned to $X_i$ in $t$. 

In Constraint Programming (see \cite{DBLP:books/daglib/0018273}), the
solution process consists 
of iteratively interleaving search phases and propagation phases. The
search phase essentially consists of enumerating all possible
variable-value combinations, until we find a solution or prove that
none exists. It is generally performed on a tree-like structure. In
order to avoid the systematic generation of all the combinations and
reduce the search space, the propagation phase shrinks the search
space: each constraint propagation algorithm removes values that a
priori cannot be part of a solution w.r.t. the partial assignment
built so far. 
The removal of inconsistent domain values is called {\it filtering}. 
%If {\it all} inconsistent values are removed from the domains with
%respect to a constraint $C$, we say that $C$ is {\it it domain consistent}.

An important modelling technique from CP are the {\it global
  constraints} that provide shorthands to often-used combinatorial
substructures.   
Global constraints embed specialized filtering techniques that exploit
 the underlying structure of the constraint to  establish stronger
 levels of consistency much more efficiently. Nowadays,
global constraints are considered to be one of the most important
components of CP solvers in practice. 

\subsection{The Among and Regular Global Constraints}

\subsubsection{Among Constraint.}This constraint restricts the number of occurrences of
some given values in a sequence of $n$ variables
\cite{DBLP:journals/jmcm/Beldi94}: 

%%%%%%%%%%%%%%%%%%%%%%%%%%%%%%%%%%%%%%%
\begin{definition}[\texttt{Among} constraint, \cite{DBLP:journals/jmcm/Beldi94}]
Let $X$=$X_1$,$\ldots$,$X_n$ be a sequence of $n$ variables, $S$
a set of values, $D^X$ the cartesian product of the variable domains in $X$. 
Let $l$ and $u$ be two integers s.t. $0 \leq l \leq u \leq n$.
$$
{\tt Among}(X,S,l,u) = \{
t \in D^X \mid l \le \, \mid \{i, t[X_i] \in S\} \mid\,  \le u \}
$$
\end{definition}
%%%%%%%%%%%%%%%%%%%%%%%%%%%%%%%%%%%%%%%

The {\tt Among} constraint can be encoded by channelling into $0$/$1$
variables using the sum constraint \cite{DBLP:conf/csclp/BessiereHHKW05}: 
$\forall i \in \{1,\ldots, n\} \;
B_i = 1 \;\leftrightarrow\; t[X_i]  \in S \;\wedge\; l \leq \sum_{i=1}^{n}
B_i \leq u.
$

\subsubsection{Regular Constraint.}
Given a deterministic finite automaton $M$ describing a regular
language, constraint  {\tt Regular}$(X,M)$ ensures that every sequence
of values taken by the variables of $X$  have to be a member of the
regular language recognised by $M$: 

%%%%%%%%%%%%%%%%%%%%%%%%%%%%%%%%%%%%%%%
\begin{definition}[{\tt Regular} constraint, \cite{regular}]
Let $M$ be a Deterministic Finite Automaton (DFA),  $\mathcal{L}(M)$
the language defined by $M$, $X$ a sequence of $n$
variables. {\tt Regular}$(X,M)$ has a solution {\it iff}
$\exists\, t \in D^X$ s.t. $t \in \mathcal{L}(M)$. 
\end{definition}
%%%%%%%%%%%%%%%%%%%%%%%%%%%%%%%%%%%%%%%

In \cite{regular}, {\tt Regular} constraint over a sequence of $n$ variables is
modelled by a layered directed graph $\mathcal{G}= (V, U)$, and a solution to 
{\tt Regular}($X$,$M$) corresponds to an $s$-$t$ path in graph
$\mathcal{G}$, where $s$ is the ``source'' node and $t$ the ``sink''
node.

\section{Modeling Sequential Patterns using CP}
\label{section:Modeling}
This section presents our CP modelling for the sequential pattern
mining problem. %Moreover, we also provide our reformulation for
%constraints presented in section \ref{sec-constraints}.
Let $I= \{i_1,i_2,\ldots,i_n\}$ be a
set of $n$ items, $EOS$ a symbol not
belonging to $I$ ($EOS \notin I)$ denoting the end of a sequence, 
$SDB$ a set of $m$ sequences and $\ell$ the maximal length of the
sequence in $SDB$.

\subsection{Modelling an Unknown Sequential Pattern}

Let $p$ be the unknown sequential pattern we are looking for. First,
$\ell$ variables $\{P_1,P_2,\ldots,P_{\ell}\}$ having $D_i$ = $I \cup \{EOS\}$
for domain are introduced to represent $p$. Second, $m$ boolean
variables $S_s$ (having $\{0,1\}$ for domain) are used such that ($S_s$
= $1$) iff sequence $s$ contains the unknown sequential pattern $p$:  
\begin{equation} \label{eq1}
(S_s=1) \Leftrightarrow (p \preceq s) 
\end{equation}

In equation (\ref{eq1}), ($S_s$ = $1$) if pattern $p$ is a \emph{subsequence}
of $s$; $0$ otherwise. So, $sup(p) = \Sigma_{s \in SDB} \, S_s$.

%The number of introduced variables is $M \times k$.   

\subsection{Modelling Sequential Pattern Mining}
\label{capturing_subsection}

Let $SDB$ be a sequence database and let a support threshold
$minsup$. To encode that ``$p \preceq s$'', we first have to generate
an automaton $A_s$ capturing all subsequences that can be found inside
a given sequence $s$. Then, we have to impose  a {\tt Regular} constraint stating
that the unknown pattern $p$ must be recognized by the automaton
$A_s$.

To reduce the number of states of the automaton, for each
sequence, we consider only its frequent items in the $SDB$
w.r.t. $minsup$. Indeed, any super-pattern of an infrequent item 
cannot be frequent. Figure \ref{auto_reification} 
shows an example of automaton generated for the third sequence of 
Table~\ref{tab:BDD}. Algorithm \ref{algo1} depicts the pseudo-code for
generating the automaton $A_s$. 

%%%%%%%%%%%%%%%%%%%%%%%%%%%%%%%%%%%%%%%%%%%%%%%%%%%%%
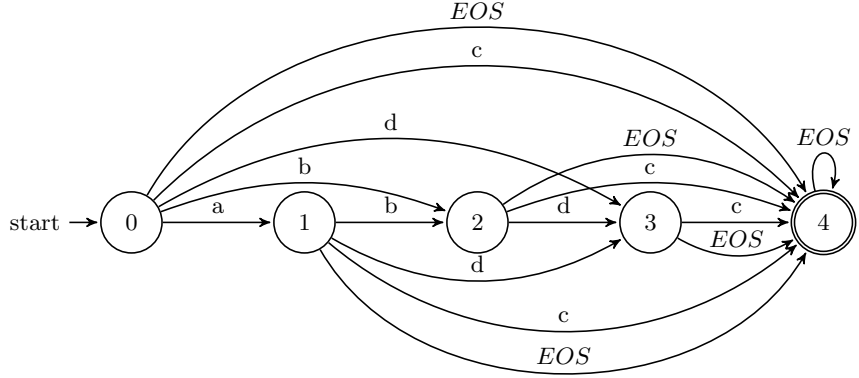
\begin{figure}[t]
\centering
\begin{tikzpicture}[->,>=stealth',shorten >=1pt,auto,node distance=2.3cm,
                    semithick]

  \node[initial,state] 		(q0)               {$0$};
  \node[state]         		(q1) [right of=q0] {$1$};
  \node[state]         		(q2) [right of=q1] {$2$};
  \node[state]         		(q3) [right of=q2] {$3$};
  \node[accepting,state]    (q4) [right of=q3] {$4$};

  \path
  (q0) 	edge     				node {a} (q1)
		edge [bend left=20]  	node {b} (q2)
        edge [bend left=30]  	node {d} (q3)
        edge [bend left=45]  	node {c} (q4)
        edge [bend left=60]  	node {$EOS$} (q4)
  (q1)  edge              		node {b} (q2)
	    edge [bend right=30] 	node {d} (q3)
	    edge [bend right=40] 	node {c} (q4)
        edge [bend right=60] 	node {$EOS$} (q4)
  (q2)  edge              		node {d} (q3)
        edge [bend left=20] 	node {c} (q4)
	    edge [bend left=35] 	node {$EOS$} (q4)
  (q3)  edge 				 	node {c} (q4)
 		edge [bend right=30] 	node {$EOS$} (q4)
  (q4) 	edge [loop above] 		node {$EOS$} (q4);
\end{tikzpicture}
\caption{The automaton modelling all subsequences of the sequence
  $\langle a~b~d~c\rangle$.\label{auto_reification}} 
\end{figure}

%%%%%%%%%%%%%%%%%%%%%%%%%%%%%%%%%%%%%%%%%%%%%%%%%%%%%%%

%To model the sequential pattern mining, we have to set the value of
%$S_s$ to $1$ if $P$ forms a sequential pattern contained in the
%sequence $s$ and to $0$ otherwise. 

To enumerate the complete set of frequent sequential patterns with
respect to a given sequence database $SDB$ and a support threshold
$minsup$, we need to express that the unknown sequential pattern $p$ 
occurs at least $minsup$ times. This problem is modelled by the
following constraints: 

\begin{theorem}[Frequent Sequential Pattern
  Mining\label{base_model_theo}] Sequential pattern mining is
  expressed by the following constraints: 
\begin{eqnarray}\label{reg_constraint}
\forall s \in SDB : S_s = 1 \leftrightarrow \mbox{{\tt Regular}}(p,A_s)
\end{eqnarray}
\begin{eqnarray}\label{f_constraint}
sup(p) = \sum_{s \in SDB} S_s  \geq minsup
\end{eqnarray}
\end{theorem}

\begin{proof}
The reified constraint (\ref{reg_constraint}) models the support constraint. By
construction, the automaton $A_s$ encodes all sequential patterns 
that are subsequences of the sequence $s$. If the {\tt Regular}
constraint is satisfied (resp. not satisfied), then $S_s$ must be
equal to $1$ (resp. must be equal to $0$). The propagation is also
performed, in a same way, from the equality constraint toward the {\tt
  Regular} constraint. 

\medskip
\noindent 
The frequency constraint (\ref{f_constraint}) enforces that at least $minsup$ variables from
$S$ must take value $1$. Together with constraint (\ref{reg_constraint}), it
enforces that at least $minsup$ sequences must support the sequential
pattern described by $p$. \hfill $\square$ 
\end{proof}

%The satisfaction of the {\tt
%  Regular} constraint implies the fact of 
%$P$ is one of the sequential pattern contained in $s$ (when satisfied)
%or not (when not satisfied). So when the contraint {\tt Regular} is
%satisfied then $S_s$ must take the value $1$ and when the contraint is
%not satisfied then $S_s$ must take the value $0$. 
%The same reasoning can be reciprocally held.  

%By definition, the set of variables $S$ represents the support of
%mined sequential patterns.  
%So, to enforced a sequence $s$ to support (resp. to not support)
%these patterns it is enough to constrained the variable $S_s$ to take
%the value $1$ (resp. $0$). 
%In the same way, controlling the values taken by $S$ allows to
%control the frequency of mined patterns. 

%%%%%%%%%%%%%%%%%%%%%%%%%%%%%%%%%%%%%%%%%%%%%%%
\begin{algorithm}[t]
\caption{Pseudo-code for generating $A_s$. \label{algo1}}
{\bf function} {\tt generateAutomaton}(Sequence $s$)\\
Automaton $A_s$;\\
$A_s$.nbState $\leftarrow$ {\tt length}($s$) + 1;\\
$A_s$.{\tt addInitialState}(0);\\
$A_s$.{\tt addAcceptingState}({\tt length}($s$));\\
\ForEach{($state$ $\in$ [0,{\tt length}($s$)])}{
	\ForEach{($position$ $\in$ [$state$+1,{\tt length}($s$)])}{
		$item$ $\leftarrow$ {\tt getItem}($s$,$position$);\\
		$A_s$.{\tt addTransition}($state$,$item$,$position$);\\
	}
	$A_s$.{\tt addTransition}($state$,$EOS$,{\tt length}($s$));\\
}
\Return $A_s$;
\end{algorithm}
%%%%%%%%%%%%%%%%%%%%%%%%%%%%%%%%%%%%%%%%%%%%%%%%%

\subsection{Modelling other Constraints}
This section shows how our CP approach enables us to express  in a
straightforward way constraints presented in
Section \ref{sec-constraints}. 

\subsubsection{Closedness Constraint.}
By definition a {\it closed} sequential pattern is the largest pattern
that is contained in all selected sequences. Intuitively, in our
encoding this corresponds to the sequential pattern having the less
number of variables $P_i$ instantiated to $EOS$. Thus, the satisfaction problem 
is turned into an optimization one:  
{\it minimize} the numbers of variables $P_i$ instantiated to
$EOS$. 
%Thus, a solution with the lowest value for the given 
%optimisation function corresponds to the largest pattern in the
%sequential database.    
To express this minimization problem, we first have to add for each
variable $P_i$ an unary constraint $c_i$ 
stating that if ($P_i$ = $EOS$) we have to pay a cost $1$; $0$
otherwise. Then, we have to minimize the cost function 
$c(p) = \sum_{P_i \in p} c_i$ to obtain a closed sequential pattern: 
%\begin{eqnarray}
%\mbox{\tt minimize}_{p} & c(p) \label{closed_constraint} \\
%sup(p)                             & \geq minsup 
%\end{eqnarray}
%
\begin{eqnarray}
\label{closed_constraint}
\begin{array}{llc}
                                        & \mbox{\tt minimize}_p & c(p) \\ 
                                        & sup(p) & \geq minsup \\
\end{array}
\end{eqnarray}
To enumerate the complete set of closed sequential patterns with
respect to a given sequence database $SDB$ and a support threshold
$minsup$, we need to avoid future patterns to be equal to the previously found closed
patterns. So, each time a frequent sequential pattern $p$ is proven closed, we dynamically 
add a new constraint to forbid it.

\subsubsection{Item Constraint.}  In order to specify that a subset of
items should or should not be present in the 
sequential patterns, we have to add the following constraint: 
\begin{eqnarray}
\mbox{\tt Among}(p,V,[l,u]) \label{contain_constraint}
\end{eqnarray}

\noindent
where $V$ is a subset of items, $l$ and $u$ are two integers s.t. $0
\leq l \leq u \leq \ell$. 
The \mbox{\tt Among} constraint enforces that the number of variables of 
$p$ that take a value from $V$ is at least $l$ and at most $u$. Since $p$
represents the sequential patterns we are looking for, the
previous constraint ensures that items of $V$ should be present
at least $l$ times in the mined patterns. 

To express the fact that the sequential patterns should not 
contain any item of $V$, we just have to set $l$ and $u$ to $0$.

\subsubsection{Size Constraint.} In order to consider the
frequent sequential patterns of size $k$, we just have to add the 
following constraints: 
\begin{eqnarray}
\forall i \in [1\ldots k] : & P_i \neq EOS \label{size&_constraint1} \\
\forall i \in [k+1\ldots\ell] :& P_i = EOS \label{size_constraint2}
\end{eqnarray}

The previous constraints
enforce that the $k$ first variables 
of $p$ must be different from $EOS$, while the $(\ell-k)$ remaining 
variables of $p$ must be equal to $EOS$. 

\medskip 
The minimum size constraint (i.e. $len(p) \ge k$) can be formulated by the following  
constraint: 
\begin{eqnarray}
\forall i \in [1\ldots k] : & P_i \neq EOS \label{size_constraint3} 
\end{eqnarray}

For the maximum size constraint (i.e. $len(p) \leq k)$, this can be modelled as follows: 
\begin{eqnarray}
\forall i \in [k+1\ldots\ell] :& P_i = EOS \label{size_constraint4}
\end{eqnarray}

%extracted patterns must have a length of $k$, the item
%$EOS$ must not been taken by the $k$ first variables of $P$ and must
%be taken by the $(l-k)$ variables of $P$. 

%%%%%%%%%%%%%%%%%%%%%%%%%%%%%%%%%%%%%%%%%%%%%%%%%%%%%%%%
\begin{figure}[t]
\centering
\begin{tikzpicture}[->,>=stealth',shorten >=1pt,auto,node distance=2.3cm,
                    semithick]

  \node[initial,state] 		(q0)               {$0$};
  \node[state]         		(q1) [right of=q0] {$1$};
  \node[state]         		(q2) [right of=q1] {$2$};
  \node[state]         		(q3) [right of=q2] {$3$};
  \node[accepting,state]    (q4) [right of=q3] {$4$};

  \path
  (q0) 	edge     				node {a} (q1)
		edge [bend left=20]  	node {b} (q2)
        edge [bend left=30]  	node {d} (q3)
        edge [bend left=45]  	node {c} (q4)
        edge [bend left=60]  	node {$EOS$} (q4)
  (q1)  %edge              		node {b} (q2)
	    edge [bend right=30] 	node {d} (q3)
	    %edge [bend right=40] 	node {c} (q4)
        edge [bend right=60] 	node {$EOS$} (q4)
  (q2)  %edge              		node {d} (q3)
        edge [bend left=20] 	node {c} (q4)
	    edge [bend left=35] 	node {$EOS$} (q4)
  (q3)  %edge 				 	node {c} (q4)
 		edge [bend right=30] 	node {$EOS$} (q4)
  (q4) 	edge [loop above] 		node {$EOS$} (q4);
\end{tikzpicture}
\caption{The new automaton modelling all subsequences of the
  sequence $\langle a~b~d~c\rangle$ satisfying the gap constraint [1,1].\label{auto_gap_reification}}
\end{figure}
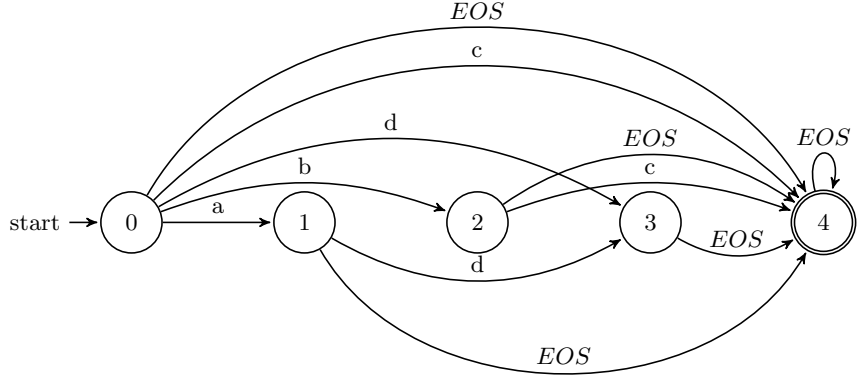
%%%%%%%%%%%%%%%%%%%%%%%%%%%%%%%%%%%%%%%%%%%%%%%%%%%%%%%%

\subsubsection{Gap Constraint.} Let $p_{[M,N]}$ be the
sequential pattern satisfying the gap constraint $[M,N]$. To encode
this constraint we have to modify the construction of the 
automaton $A_s$ (cf. Theorem \ref{base_model_theo}) in a
such way that only transitions respecting the gap constraint will be
kept. Let $A_s^{[M,N]}$ be the new resulting automaton representing
all  sequential patterns that are subsequence of the sequence s and
satisfying the gap constraint. Finally, the reified constraint
(\ref{reg_constraint}) is rewrited as follows:
\begin{eqnarray}
\forall s \in SDB : S_s = 1 \leftrightarrow \mbox{{\tt Regular}}(p,A_s^{[M,N]}) \label{gap_constraint}
\end{eqnarray} 

\begin{theorem}[Frequent Sequential Pattern Mining with Gap\label{model2_theo}]
Sequential pattern mining with a gap constraint is expressed by
constraints (\ref{f_constraint}) and (\ref{gap_constraint}).  
\end{theorem}

%%\begin{proof} A direct consequence of our modelling. \hfill $\square$ 
%It follows the proof of the theorem \ref{base_model_theo}.
%%\end{proof}

Figure \ref{auto_gap_reification} gives the new automaton obtained
from the automaton of Figure \ref{auto_reification} with the gap
constraint $[1,1]$. For instance, the sequential pattern $\langle
a~b\rangle$ does not satisfy the gap constraint; it is not 
recognized by the new automaton. 

To generate the automaton $A_s^{[M,N]}$ for a sequence $s$, we need to
modify Algorithm~\ref{algo1} in a such way that only valid transitions satisfying 
the gap constraint $[M,N]$ are considered . This is done by adding the following condition inside the
second loop foreach: 
$
(state == 0 \mid\mid M \leq position - state \leq N)
$ (see Algorithm~\ref{algo1}).

\subsubsection{Regular Expression Constraint.} Let $A_{re}$ be an
automaton encoding a regular expression over the set of items. Then,
the regular expression constraint can be formulated as follows: 
\begin{eqnarray}
\mbox{\tt Regular}(p,A_{re}) 
\end{eqnarray}

As presented in \cite{DBLP:journals/tcs/ChangP97}, a regular expression can be translated
into a deterministic finite automaton. Thus, the {\tt Regular}
constraint over $p$ ensures that every sequence of values taken by the
variables of $p$ have to be a member of the regular language
recognised by $A_{re}$, therefore recognized by the regular expression
associated to $A_{re}$.

%%%%%%%%%%%%%%%%%%%%%%%%%%%%%%%%%%%%%%%%%%%%%%%%%
%\begin{algorithm}[ht]
%\caption{Pseudo-code for generating $A_s^{[M,N]}$ \label{algo2}}
%{\bf function} {\tt generateAutomaton}(Sequence $s$, Gap [$M,N$])\\
% Automaton $A_s$;\\
%$A_s$.nbState $\leftarrow$ {\tt length}($s$) + 1;\\
%$A_s$.{\tt addInitialState}(0);\\
%$A_s$.{\tt addAcceptingState}({\tt length}($s$));\\
%\ForEach{$state$ $\in [0$,{\tt length}$(s)]$}{
%	\ForEach{$position \in [state+1$,{\tt length}$(s)]$}{
%		\If{$(state = 0 \mid\mid M \leq position - state \leq N)$}{
%		$item$ $\leftarrow$ {\tt getItem}($s$,$position$);\\
%		$A_s$.{\tt addTransition}($state$,$item$,$position$);\\
%		}
%	}
%	$A_s$.{\tt addTransition}($state$,$EOS$,{\tt length}($s$));\\
%}
%\Return $A_s$;
%\end{algorithm}
%%%%%%%%%%%%%%%%%%%%%%%%%%%%%%%%%%%%%%%%%%%%%%%%%%%%%

\section{Related Work}
\label{section:relatedwork}
{\bf Computing Sequential Patterns.} 
In the context of constraint-based sequential pattern mining, several
algorithms have been
proposed~\cite{DBLP:conf/edbt/SrikantA96,DBLP:journals/ml/Zaki01,DBLP:conf/icde/PeiHPCDH01,DBLP:conf/sdm/YanHA03,DBLP:conf/icde/WangH04,DBLP:journals/tkde/GarofalakisRS02}. All
these algorithms exploit properties of the constraints (i.e.,
monotonicity, anti-monotonicity or succinctness) to perform 
effective pruning. For constraints that do
not fit in these categories, they are handled by relaxing them to
achieve some nice property (like anti-monotonicity) facilitating the
pruning. For instance, Garofalakis et
al.~\cite{DBLP:journals/tkde/GarofalakisRS02} proposed
regular expressions as constraints and developed a family of
\texttt{SPIRIT} algorithms each achieving a different kind of
relaxation of the regular expression constraint. Such a method,
though interesting, makes tricky the integration of such  
constraints in a nice and elegant way. So, unlike these algorithms,
our approach enables to address in a unified framework a broader set
of constraints, and more importantly, to combine them simultaneously.

\medskip
\noindent
{\bf CP for Pattern Mining.} 
In the context of local patterns, an approach using CP for itemset
mining has been proposed in  
\cite{DBLP:conf/kdd/RaedtGN08}. This approach addresses in a unified
framework a large set of local patterns and constraints such as
frequency, closedness, maximality, constraints that are monotonic or
anti-monotonic. To deal with richer
patterns satisfying properties involving several local patterns,
different extensions have been proposed, such as pattern
sets~\cite{DBLP:journals/tkde/GunsNR13}, n-ary
patterns~\cite{CI-KHIARI-10b}, top-k
patterns~\cite{DBLP:conf/dis/UgarteBLC12} or
skypatterns~\cite{DBLP:conf/f-egc/UgarteBLCL13}.
Our approach also benefits from the recent progress on
cross-fertilization between data mining and CP for itemset
mining, but it addresses a different problem with a different
modelling. 

\medskip
\noindent
{\bf CP for Sequence Mining.} More recently, Coquery et
al.~\cite{DBLP:conf/ecai/CoqueryJSS12} have 
proposed a SAT-Based approach for Discovering frequent, closed and
maximal sequential patterns with wildcards in only a single sequence of items. 
However, unlike \cite{DBLP:conf/ecai/CoqueryJSS12}, our approach considers
a database of sequences of items. Moreover, in
\cite{DBLP:conf/ecai/CoqueryJSS12}, the sequential patterns with
non-contiguous items are modelled using empty items as wildcards. But the gap between the items have to be fixed. Then,  
for instance the two sequential patterns $\langle a~o~b\rangle$ and $\langle
a~o~o~b\rangle$ are considered different. 
On the contrary, our modelling enables us to define any (minimal or maximal) value for the gap. 

%This is not the case in our
%modelling where these two sequential patterns represent the same
%pattern $\langle a~b\rangle$. 

\section{Experimentations}
\label{section:experimentations}
Experiments are conducted on texts from biological and medical
texts. The goal is to discover relations between genes and rare
diseases. The details of this application is given
in~\cite{BCCC2012cbms}. In this 
section, we focus on the extraction of sequential patterns using our
CP approach, and we give quantitative results showing the relevant of
the approach.  

\subsection{Case Study}

\subsubsection{Settings.}
We created a corpus from the PubMed database
using HUGO\footnote{www.genenames.org} dictionary and Orphanet
dictionary to query the database to 
get sentences having at least one rare disease and one gene. $17,527$ sentences
have been extracted in this way and we labelled the gene and rare
disease (RD) names thanks to the two dictionaries.  For instance, the sentence
``\textit{\textbf{$<$disease$>$Muir-Torre
    syndrome$<${\tt\char`\\}disease$>$} is usually inherited in an
  autosomal dominant fashion and associated with mutations in the
  mismatch repair genes, predominantly in
  \textbf{$<$gene$>$MLH1$<${\tt\char`\\}gene$>$} and
  \textbf{$<$gene$>$MSH2$<${\tt\char`\\}gene$>$} genes.}'' contains
one recognized RD, and two recognized genes.    
These $17,527$ sentences are the training corpus from which we
experiment the sequential pattern extraction.

\subsubsection{Sequential Pattern Extraction.}
Sequences of the SDB are the sentences of the training corpus: an item
corresponds to a word of the sentence. We carry out a POS tagging of
the sentences thanks to the TreeTagger tool~\cite{treetagger}. In the
sentences, each word is replaced by its lemma, except for gene names
(respectively disease names) which are replaced by the generic item
$GENE$ (respectively $DISEASE $). Note that unlike machine learning
based methods, our approch does not require to annotate the relations:
they are discovered.  

In order to discover sequential patterns, we use usual constraints 
such as \textit{the minimal frequency} and \textit{the minimal length}
constraints and other useful constraints expressing some \textit{linguistic
knowledge} (e.g. \textit{membership} and \textit{association}
constraints). The goal is to retain sequential
patterns which convey linguistic regularities (e.g., gene-rare disease
relationships). Our method offers a natural way to simultaneously combine
in a same framework these constraints coming from various origins. We
briefly sketch these constraints. 
%The algorithm was running with the following different constraint
%characteristics to extract the closed sequential patterns. 

\noindent 
{\footnotesize $\bullet$} \textit{The minimal frequency constraint.} Three values of minimal frequency have been experimented:
$2\%$, $5\%$, and $10\%$.\\
%{\footnotesize $\bullet$} \textit{The gap.} We have conducted experiments
%without and with gap value (chosen empirically at [0,10]).\\
%{\footnotesize $\bullet$} \textit{The maximal scope.} We set a maximal
%scope value at 20 to reduce the number of extracted patterns: we assume
%that the maximal number of itemsets between the first itemset and the last
%itemset of patterns having gene--RD relationships is almost 20
%(corresponding to 20 words in the sentence).\\
{\footnotesize $\bullet$} \textit{The minimal length constraint.} The aim of this
constraint is to remove sequential patterns that are too small w.r.t. 
the number of items (number of words) to be relevant
linguistic patterns. We tested this constraint with a value set to $3$. \\
{\footnotesize $\bullet$} \textit{The membership constraint.} This constraint
enables to filter out sequential patterns that do not contain some
selected items. For example, we express that the extracted patterns
must contain at least three items (expressing the linguistic
relation): $GENE$, $DISEASE$ and noun or verb\footnote{For each word
  (i.e. item), its grammatical category is stored in a base.}. We used
the item constraint to enforce this constraint.\\
{\footnotesize $\bullet$} \textit{The association constraint.} This constraint
expresses that all sequential patterns that contain the verb item must contain 
its lemma and its grammatical category. We used the  item constraint to enforce 
this constraint. \\ 
%As a pretreatment, all verbs are collected. We used an item constraint to
%express an association constraint. Since, only lemmas are kept in the dataset,
%we can enforced the association constraint by enforcing one of the pattern
%variables to take its value from the set of verbs (used of $S$ in the {\tt
%Among} contraint).
%For example, pattern $\langle (mutation\, NNS)~
%(in\, IN)~(isocitrate\-\, NN)~\-(GENE)~(occur\, VB)~(DISEASE)\rangle$ is correct
%with respect to the association constraint, whereas $\langle (NNS)~
%(IN)(NN)~(GENE)~(VB)~(DISEASE)\rangle$ is not correct.
{\footnotesize $\bullet$} \textit{The closedness constraint.} In order to exclude redundancy between patterns, we used
{\it closed} sequential patterns.

\subsubsection{Experimental Protocol.} 
We carried out experiments on several subsets of the PubMed dataset
with different sizes ranging from $50$ to $500$ sentences.  
A timeout of $10$ hours has been used. 
For each subset, we report the number of extracted closed sequential
patterns and the CPU-times to extract them (in seconds). When the
timeout is reached, the number of extracted patterns (until the
timeout) is given in parenthesis. 

All experiments were conducted on AMD Opteron $2.1$ GHz and a RAM of
$256$ GB.  We implemented our proposal in C++ using the library
toulbar2\footnote{\tt
  https://mulcyber.toulouse.inra.fr/projects/toulbar2.} for solving
constrained optimization problems modelled as {\it Cost Function Network} (CFN). 

\subsection{Results}

\begin{table}[t]
\centering
{\small
\begin{tabular}{|l|r|r|r|r|r|r|r|r|r|r|}
\hline
\#sentences &   \multicolumn{2}{c|}{50} & \multicolumn{2}{c|}{100} & \multicolumn{2}{c|}{150} & \multicolumn{2}{c|}{200} & \multicolumn{2}{c|}{250}\\ 
     	     &   \#sol. & time  &   \#sol. & time  &   \#sol. & time&   \#sol. & time&   \#sol. & time\\
\hline
freq $>$ 2\% & 129 & ~~1,105 & 329 & ~12,761 & 441 & ~35,164 & (89) & -- & (34) & --\\
\hline
freq $>$ 5\% & 47 & 285 & 67 & 1,571 & 81 & 2,091 & 94 & ~~4,119 & 119 & ~~8,516\\
\hline
freq $>$ 10\%& 4 & 53 & 21& 251 & 26 & 577 & 29 & 1,423 & 28 & 2,764 \\  
\hline
\end{tabular}

\medskip

\begin{tabular}{|l|r|r|r|r|r|r|r|r|r|r|}
\hline
\#sentences &   \multicolumn{2}{c|}{300} & \multicolumn{2}{c|}{350} & \multicolumn{2}{c|}{400} & \multicolumn{2}{c|}{450} & \multicolumn{2}{c|}{500}\\  
     	     &   \#sol. & time  &   \#sol. & time  &   \#sol. & time&   \#sol. & time&   \#sol. & time\\
\hline
freq $>$ 2\% & (129) & -- & (45) & -- & (10) & -- & (1) & -- & (0) & --\\
\hline
freq $>$ 5\% & 101 & ~~9,620 & 93 & ~16057 & 83 & ~21,764 & 84 & ~35,962 & (26) & --\\
\hline
freq $>$ 10\%& 30 & 5147 & 24 & 4,493 & 23 & 7,026 & 20 & ~13,744 & 21 & ~17,708 \\  
\hline
\end{tabular}
}
\caption{Results obtained on different subsets of the PubMed dataset.\label{result-tab}}
\end{table}

Table~\ref{result-tab} reports the results we obtained on
different subsets of the PubMed dataset with values of $minsup$
ranging from $2\%$ to $10\%$. From these results, we can
draw the following remarks. 

\medskip
\noindent
{\bf i) Soundness and Flexibility}. As the resolution performed by the CP
solver is sound and complete, our approach is able to mine the correct
and complete set of sequential patterns satisfying the different
constraints. We compared the sequential patterns extracted by our
approach with those found by \cite{BCCC2012cbms}, and the two
approaches return the same set of patterns. Table~\ref{result-tab}
depicts the number of closed sequential patterns according to
$minsup$. As expected, the lower $minsup$ is, the larger the
number of extracted sequential patterns. 

\medskip
\noindent
{\bf ii) Highlighting useful sequential patterns}. Our approach
allowed to extract several relevant {\it linguistic patterns}. 
%Extracted patterns are qualitatively interesting.
%They formed valid linguistic patterns.
Such patterns can be used to explain RD-gene relationships from PubMed
articles. For instance, three sequential patterns of great interest
were highlighted by the expert: 
%Among the extracted patterns we can look at the three following ones:
\begin{enumerate}
\item $\langle (DISEASE)~(be)~(cause)~(by)~(mutation)~(in)~(the)~(GENE)\rangle$
\item $\langle (GENE)~(occur)~(in)~(DISEASE)\rangle$
\item $\langle (DISEASE)~(be)~(an)~(mutation)~(in)~(GENE)\rangle$
\end{enumerate}

From a biomedical point of vue, these sequential patterns are
interesting since they convey a notion of {\it causality} (i.e. gene
{\it cause} rare disease). 

\medskip
\noindent
{\bf iii) Computational Efficiency}. This experiment quantify
runtimes and the scalability of our approach. In practice, runtimes
vary according to the size of the datasets. For datasets with size up
to $200$, the set of all solutions is computed. We observe that
runtimes vary from few seconds for high frequency thresholds to about
few hours for low frequency thresholds. However, for large size
datasets ($\ge 200$) and low frequency thresholds (i.e. $minsup =
2\%$), the CP approach does not succeed to complete the extraction of
all closed sequential patterns within a timeout of $10$ hours. Indeed,
with the increase of the size of the dataset, the search space and the
runtime increase drastically, and the solver spends much more time to
find the first solution. %For instance, for the dataset with $500$ sentences and
%with $minsup = 2\%$, it takes about $12$ hours to find the first
%solution, and $52$ solutions were found after $24$ hours of
%computation.

Finally, note that comparing runtimes with those obtained by 
ad hoc approaches would be rather difficult. In fact, these approaches
use devoted techniques and do not offer the same level of genericity
and expressivity as in our CP approach. Moreover, they cannot push in
depth simultaneously different categories of constraints. Above all,
there does not exist any algorithm, neither tool, for extracting
sequential patterns under all the contraints proposed in our  work.

\section{Conclusion}
\label{section:conclusion}
We have proposed a flexible approach to mine sequential patterns of
items in a sequence database. The declarative side of the CP
framework easily enables us to address a large set of constraints and
leads to a unified framework for extracting sequential patterns under
constraints. Finally, the feasibility and the interest of our approach has
been highlighted through experiments on a case study in biomedical
literature for discovering gene-RD relations from PubMed articles.  

We are currently investigating a new direction to enhance the
efficiency of our approach: instead of constructing an automaton for
every sequence, it would be more efficient to build some variant of a
Prefix Tree Automata on the original dataset to avoid some
redundancies. Furthermore, we intend to extend our approach to discover sequential
patterns of itemsets in a sequence database. 
Discovering pattern sets is an attractive road to propose actionable
patterns and the CP modelling is a proper paradigm to tackle this
challenge \cite{DBLP:journals/tkde/GunsNR13,CI-KHIARI-10b}. Further
work is to address this issue. 

\medskip
\noindent
{\it Acknowledgements}. This work is partly supported by the ANR
(French Research National Agency) funded project FiCOLOFO
ANR-10-BLA-0214. We would like to thank Bruno Cr\'emilleux for his
valuable comments.

\bibliographystyle{plain}
{
\small
\bibliography{biblio}
}

\end{document}